\documentclass{article}

\usepackage{microtype}
\usepackage{graphicx}
\usepackage{booktabs}
\usepackage{mathrsfs}

\usepackage{thm-restate}
\usepackage{comment}
\usepackage{nicefrac}
\usepackage[usenames,dvipsnames]{xcolor}
\usepackage[format=hang]{subcaption}

\usepackage{amsmath}
\usepackage{amssymb}
\usepackage{amsthm}
\usepackage{dsfont}
\theoremstyle{definition}
\newtheorem{theorem}{Theorem}

\newtheorem{definition}[theorem]{Definition}

\newtheorem{lemma}[theorem]{Lemma}

\usepackage{enumitem}
\setlist[enumerate]{leftmargin=0.5cm,topsep=0pt,itemsep=-2pt}
\setlist[itemize]{leftmargin=0.5cm,topsep=0pt,itemsep=-2pt}

\usepackage{mathrsfs}

\usepackage{mathtools}
\usepackage{hyperref}
\hypersetup{
    colorlinks=true,
    linkcolor=blue,
    citecolor=cyan,
}

\usepackage{tikz}
\usetikzlibrary{bayesnet}

\pgfdeclarelayer{edgelayer}
\pgfdeclarelayer{nodelayer}
\pgfsetlayers{edgelayer,nodelayer,main}

\definecolor{hexcolor0xbfbfbf}{rgb}{0.749,0.749,0.749}

\tikzset{>=latex}
\tikzstyle{none}   = [inner sep=0pt]
\tikzstyle{line}   = [ -, thick, shorten <=1pt, shorten >=1pt ]
\tikzstyle{arrow}  = [ ->, thick, shorten <=1pt, shorten >=1pt ]
\tikzstyle{ardash} = [ dashed, ->, thick, shorten <=1pt, shorten >=1pt ]

\tikzstyle{empty}=[circle,opacity=0.0,text opacity=1.0,inner sep=0pt]
\tikzstyle{box}=[rectangle,fill=White,draw=Black]
\tikzstyle{filled}=[circle,thick,fill=hexcolor0xbfbfbf,draw=Black]
\tikzstyle{hollow}=[circle,thick,fill=White,draw=Black]
\tikzstyle{param}=[rectangle,fill=Black,draw=Black,inner sep=0pt,minimum width=4pt,minimum height=4pt]
\tikzstyle{paramhollow}=[rectangle,thick,fill=White,draw=Black,inner sep=0pt,minimum
width=4pt,minimum height=4pt]

\usepackage[accepted]{icml2021}

\icmltitlerunning{Beyond Verifiable Rewards: Scaling Reinforcement Learning for Language Models to Unverifiable Data}

\begin{document}

\twocolumn[
\icmltitle{Beyond Verifiable Rewards:\\ Scaling Reinforcement Learning for Language Models to Unverifiable Data}

\begin{icmlauthorlist}
\icmlauthor{Yunhao Tang}{meta1}
\icmlauthor{Sid Wang}{meta1}
\icmlauthor{Lovish Madaan}{meta1,ucl1}
\icmlauthor{R\'emi Munos}{meta2}
\end{icmlauthorlist}
\icmlaffiliation{meta1}{Meta GenAI}
\icmlaffiliation{meta2}{Meta FAIR}
\icmlaffiliation{ucl1}{University College London}

\icmlkeywords{Machine Learning, ICML}

\vskip 0.025\linewidth
]

\printAffiliationsAndNotice{} 

\begin{abstract}
We propose to scale RL to unverifiable data with a novel algorithm JEPO (\textbf{J}ensen's \textbf{E}vidence lower bound \textbf{P}olicy \textbf{O}ptimization). While most prior efforts on scaling RL for LLMs focus on verifiable data where ground truth answers are typically short-form and can be matched easily; we investigate the case where such assumptions are less valid (e.g., when answers are long-form such as mathematical proofs). To scale RL training to unverifiable data with contemporary training constraints, we propose JEPO. JEPO applies Jensen's evidence lower bound, a pragmatic simplification of the evidence lower bound which views chain-of-thought as a latent variable in the generative process. We show that on verifiable data (math), JEPO is as effective as RL with verifiable rewards; on semi-verifiable data (numina), JEPO improves on soft-match based evaluations compared to RL with verifiable rewards which can only leverage a subset of the data source; finally, on unverifiable data (numina-proof), JEPO outperforms SFT and a few ablation baselines on likelihood evaluations.
\end{abstract}

\section{Introduction}

Reinforcement learning from verifiable rewards (RLVR) has proved effective at endowing language models with capabilities beyond canonical pre-training and supervised fine-tuning \citep{jaech2024openai,shao2024deepseekmath,lambert2024t,guo2025deepseek,team2025kimi,su2025expanding}. At its core, reinforcement learning (RL) allows for the optimization of chain-of-thought at scale, which elicits significant performance improvements especially for reasoning intensive tasks \citep{ling2017program,wei2022chain}. In the case of mathematical reasoning, it encourages step-by-step solutions that lead up to a final answer \citep{cobbe2021training,lightman2023let}, where correctness can be verified to produce a reward signal for RL training.

However, a main limitation of current RLVR is the data source: verifiable rewards are mostly derived from datasets where ground truth answers are short-form and can be checked in relatively easy ways \citep{guo2025deepseek,team2025kimi,su2025expanding}. For example, most answers to popular benchmarks are integers and short expressions \citep{hendrycks2021measuring,aime_1983_2024}. This practical limitation makes it hard to scale RL to more general datasets where answer correctness is hard to check. For instance, for long-form mathematical data where the answer is the whole proof, its inherent correctness is hard to assess without expert human evaluations \citep{petrov2025proof}.

\begin{figure*}[t]
\centering
\includegraphics[width=7in]{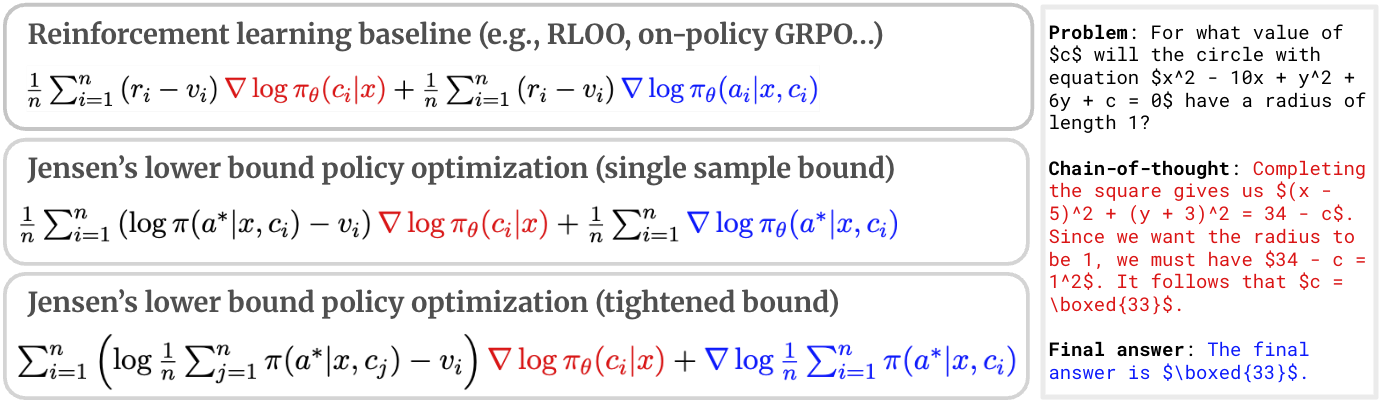}
\caption{\small{A canonical RL algorithm updates both its chain-of-thought policy $\pi_\theta(c|x)$ and the final conclusion $\pi_\theta(a|x,c)$ with advantage function computed from reward $r_i$ and an optional baseline $v_i$. JEPO has similar counterparts: updating the chain-of-thought policy using likelihood scores as the effective reward, and updating the answer policy using a supervised loss. Unlike RL baselines, JEPO does not require access to a reward $r_i$ but only access to a ground truth answer $a^*$. Due to the implementation-level similarity between JEPO and RL, it is straightforward to incorporate JEPO into existing stacks of large-scale RL training. We use the same baseline notation for the RL and JEPO loss, though they differ in practice. In general $v_i$ can be a leave-one-out control variate that is computed from other $n-1$ samples in the batch.}}
\label{figure:algorithm}
\end{figure*}

 The boundary between verifiable and unverfiable data, though often blurry in practice, can be made more actionable: we define data as unverifiable, if its ground truth answer cannot be verified with a reasonably simple automatic procedure. Naturally, it is of interest to scale RL to such data sources, for a few notable reasons: (1) some data have inherently long answers which cannot be cast into short-form answers in a straightforward way; (2) data sources with long-form answers exist in abundance, and it is sub-optimal not to leverage such data for training. In this work, we seek to tackle the problem of scaling RL to unverifiable data.

We propose JEPO (\textbf{J}ensen's \textbf{E}vidence lower bound \textbf{P}olicy \textbf{O}ptimization), a novel RL algorithm that can equally post-train on verifiable or unverifiable data. The design of the algorithm is inspired by a latent variable view of chain-of-thought \citep{hoffman2024training,hu2024unveiling}. As a major algorithmic innovation, contrast to prior work, we make use of \emph{Jensen's evidence lower bound}, a novel pragmatic simplification of the full evidence lower bound \citep{blei2006variational,blei2017variational} named after Jensen's inequality \citep{jensen1906fonctions}. Optimizing such a simplified objective forgoes the need of training expensive auxiliary models, making JEPO more suitable for contemporary large-scale training \citep{brown2020language,achiam2023gpt}. 

The final algorithm consists of an hybrid RL and supervised learning loss. As a major advantage over online RL baselines, JEPO does not require any external verifiable reward, lifting the requirement that ground truth be easily verifiable. JEPO also shares much implementation-level similarity with online RL algorithms, making it easy to integrate into an existing large-scale workflow. See Figure~\ref{figure:algorithm} for a visual depiction of the similarity and difference between JEPO and RL baselines. In more details, our technical contributions are as follows:
\begin{itemize}
    \item (\textbf{Algorithm}) Followed by a brief background on latent variable modeling, we derive the Jensen's evidence lower bound in Section~\ref{sec:algo}. In Section~\ref{sec:algo-multi}, we show how its multi-sample extension \citep{burda2015importance} tightens the theoretical bound and alludes to better performance in practice. For all objectives, we derive stochastic optimization algorithms that can be practically implemented.
    \item (\textbf{Theoretical connections}) We draw insightful connections between the full ELBO, RL and JEPO in Section~\ref{sec:connections}. We discuss a few different connections of interest to readers from different backgrounds, such as the practical trade-offs of RL vs. JEPO. See Figure~\ref{figure:cotgraph} for an illustration of the graphical models connecting JEPO and probabilistic inference.
    \item (\textbf{Implementation}) In Section~\ref{sec:implement}, we highlight practical implementation details that make JEPO work the best, highlighting the fact that the resulting algorithm takes a similar form to common RL algorithms for LLM. This means that JEPO is easy to integrate into an existing workflow. See Figure~\ref{figure:algorithm} for a summarized comparison. 
    \item (\textbf{Experiments}) Finally in Section~\ref{sec:exp-verifiable}, Section~\ref{sec:exp-semi} and Section~\ref{sec:exp-unverifiable}, we show that for verifiable data, JEPO is competitive compared to online RL with verifiable reward. For semi-verifiable and unverifiable data, JEPO has performance advantage over online RL, SFT or other ablation baselines. As a by-product, we showcase the utility of generating chain-of-thought for long-form proofs, an observation that is interesting in its own right.
\end{itemize}

\section{Reinforcement learning for language models}

A language model can be understood as a policy $\pi_\theta$ in the context of reinforcement learning. Given a prompt $x$, the policy generates a response $y$, which then gets assessed by a human user. Usually, the objective is to optimize $\pi_\theta$ such that certain reward function $r(x,y)$ that captures human preference is maximized \citep{christiano2017deep,ouyang2022training}. Formally, consider the maximization problem
\begin{align}
    \max_\theta \mathbb{E}_{y\sim \pi_\theta(\cdot|x)}\left[r(x,y)\right] -\beta \mathbb{KL}\left(\pi_\theta(\cdot|x),\pi_\text{ref}(\cdot|x)\right) \label{eq:rlhf}
\end{align}
with a KL regularization that encourages $\pi_\theta$ to stay close to the reference policy. The reward $r(x,y)$ captures the human preference of response $y$ in response to prompt $x$ and can take various forms: for example, it can be extracted from human annotations \citep{christiano2017deep,ziegler2019fine,ouyang2022training}, computed using automatic feedback such as code execution \citep{gehring2024rlef,wei2025swe}. We focus on a specialized setting where the reward is derived from access to a certain \emph{ground truth} of the problem.

\subsection{RL from ground truth feedback}

We focus on applications where the prompt $x$ typically specifies a question and there is an example of a desirable ground truth $a^\ast$. Such a formulation is applicable to mathematical reasoning \citep{hendrycks2021measuring,uesato2022solving,lightman2023let} where $x$ is a question and $a^\ast$ is the ground truth answer. When the correctness of the model generated answer $a$ can be easily verified against the ground truth $a^*$, a verifiable reward $r$ is available by matching $a^\ast$ against the answer $a$. As another example, when $a^*$ is a long-form proof, such a reward is not immediately available and such cases are considered less verifiable.

In broader context, RLVR also includes code applications where the reward is computed via unit tests \citep{gehring2024rlef,wei2025swe}. We do not consider such use cases.

\subsection{Chain-of-thought}

For aforementioned applications where the model is required to reason about the question $x$ and generate an answer $a$, we get the model to generate chain-of-thoughts - a sequence of reasoning steps $c$ leading up to the final conclusion \citep{ling2017program,wei2022chain}. Henceforth, we can decompose the generation $y=(c,a)$ into a chain-of-thought $c$ and an answer $a$. The generative process for the response $y\sim \pi_\theta(\cdot|x)$ is made more concrete as
\begin{align}
    c\sim \pi_\theta(\cdot|x), a\sim \pi_\theta(\cdot|x,c).\label{eq:cot-generative}
\end{align}

Given a prompt $x$, the intuitive role of chain-of-thought is such that it makes the \emph{marginal} likelihood of the ground truth answer $a^\ast$ higher. As such, we can interpret chain-of-thought as a latent variable and formulate the optimization of chain-of-thought as latent variable modeling  \citep{hu2024unveiling,hoffman2024training}.

\begin{figure*}[!t]
\centering
\subcaptionbox{\small{Probabilistic inference}}[.25\linewidth]{
\begin{tikzpicture}
	\begin{pgfonlayer}{nodelayer}
		\node [style=filled] (0) at (0, 0) {$o$};
		\node [style=hollow] (1) at (0, 1.5) {$z$};
		\node [style=box] (4) at (-1.2, 1.5) {$\theta$};
		\node [style=box] (5) at (1.2, 1.5) {$\phi$};
		\node [style=empty] (10) at (0, 0.66) {};
		\node [style=empty] (11) at (-0.5, 0.66) {};
		\node [style=empty] (12) at (0.5, 0.66) {};
	\end{pgfonlayer}
	\begin{pgfonlayer}{edgelayer}
		\draw [style=arrow] (1) to (0);
		\draw [style=arrow] (4) to (1);
		\draw [style=arrow] (4) to (0);
		\draw [style=arrow][bend right=60,dashed] (0) to (1);
		\draw [style=arrow][dashed] (5) to (1);
	\end{pgfonlayer};
\end{tikzpicture}
}
\subcaptionbox{\small{CoT with full ELBO}}[.23\linewidth]{
\begin{tikzpicture}
	\begin{pgfonlayer}{nodelayer}
		\node [style=filled] (0) at (0, 0) {$a^*$};
		\node [style=hollow] (1) at (0, 1.5) {$c$};
		\node [style=box] (4) at (-1.2, 1.5) {$\theta$};
		\node [style=box] (5) at (1.2, 1.5) {$\phi$};
		\node [style=empty] (10) at (0, 0.66) {};
		\node [style=empty] (11) at (-0.5, 0.66) {};
		\node [style=empty] (12) at (0.5, 0.66) {};
	\end{pgfonlayer}
	\begin{pgfonlayer}{edgelayer}
		\draw [style=arrow] (1) to (0);
		\draw [style=arrow] (4) to (1);
		\draw [style=arrow] (4) to (0);
		\draw [style=arrow][bend right=60,dashed] (0) to (1);
		\draw [style=arrow][dashed] (5) to (1);
	\end{pgfonlayer};
\end{tikzpicture}
}
\subcaptionbox{\small{CoT with Jensen's bound}}[.23\linewidth]{
\begin{tikzpicture}
	\begin{pgfonlayer}{nodelayer}
		\node [style=filled] (0) at (0, 0) {$a^*$};
		\node [style=hollow] (1) at (0, 1.5) {$c$};
		\node [style=box] (4) at (-1.2, 1.5) {$\theta$};
		\node [style=box] (5) at (1.2, 1.5) {$\theta$};
		\node [style=empty] (10) at (0, 0.66) {};
		\node [style=empty] (11) at (-0.5, 0.66) {};
		\node [style=empty] (12) at (0.5, 0.66) {};
	\end{pgfonlayer}
	\begin{pgfonlayer}{edgelayer}
		\draw [style=arrow] (1) to (0);
		\draw [style=arrow] (4) to (1);
		\draw [style=arrow] (4) to (0);
		\draw [style=arrow][dashed] (5) to (1);
	\end{pgfonlayer};
\end{tikzpicture}
}
\subcaptionbox{\small{CoT with Jensen's bound with KL regularization}}[.23\linewidth]{
\begin{tikzpicture}
	\begin{pgfonlayer}{nodelayer}
		\node [style=filled] (0) at (0, 0) {$a^*$};
		\node [style=hollow] (1) at (0, 1.5) {$c$};
		\node [style=box] (4) at (-1.2, 1.5) {$\theta_\text{ref}$};
		\node [style=box] (5) at (1.2, 1.5) {$\theta$};
		\node [style=empty] (10) at (0, 0.66) {};
		\node [style=empty] (11) at (-0.5, 0.66) {};
		\node [style=empty] (12) at (0.5, 0.66) {};
	\end{pgfonlayer}
	\begin{pgfonlayer}{edgelayer}
		\draw [style=arrow] (1) to (0);
		\draw [style=arrow] (4) to (1);
		\draw [style=arrow][dashed] (5) to (1);
		\draw [style=arrow] (5) to (0);
	\end{pgfonlayer};
\end{tikzpicture}
}
\caption{\small{Graphical models for various algorithmic formulations discussed in this work. Solid lines represent generative models and dashed lines represent inference models. Circles represent random variables and squares represent parameters. Shading indicates that the random variable is observed, and is used for providing feedback for the learning process. For CoT optimization, $a^*$ is a simplified notation for the binary optimality variable $\mathds{1}_{\{a=a^\ast\}}$ from the random variable $a$. See Appendix~\ref{appendix:graph} for a more detailed explanation.}}
\label{figure:cotgraph}
\end{figure*}
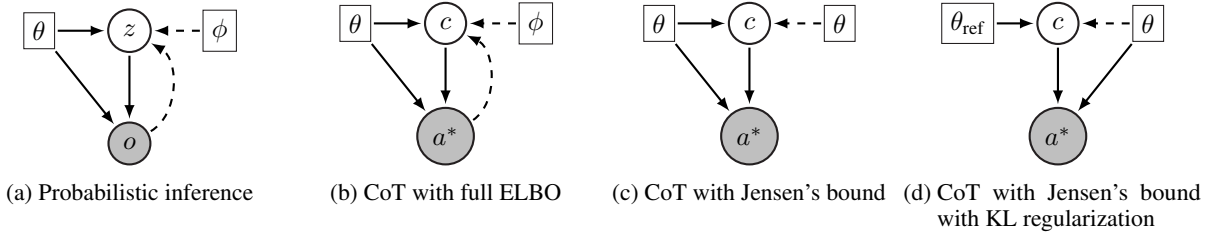

\section{Jensen's lower bound for chain-of-thought as latent variable modeling}\label{sec:algo}

We start with the initial motivation to increase the marginal likelihood of the ground truth answer $a^\ast$ (i.e., the evidence) given the generative process in Eqn~\eqref{eq:cot-generative}
\begin{align}
    \max_\theta \log \pi_\theta(a^\ast|x).\label{eq:likelihood-objective}
\end{align}
Directly optimizing the log likelihood 
is not tractable because its gradient cannot be estimated via samples in an unbiased way (see, e.g., discussion on this in the probabilistic inference literature \citep{blei2017variational}). As the main contribution of this work, we propose a tractable lower bound objective by directly applying the Jensen inequality to lower bound the log likelihood
\begin{align}
    \log \pi_\theta(a^\ast|x) &= \log \mathbb{E}_{c\sim \pi_\theta(\cdot|x)}\left[\pi_\theta(a^\ast|x,c)\right] \notag \\
    &\geq \underbrace{\mathbb{E}_{c\sim \pi_\theta(\cdot|x)}\left[\log 
\pi_\theta(a^\ast|x,c)\right]}_{\mathcal{L}_\theta(x,a^\ast)},\label{eq:simple-lower-bound}
\end{align}
where we exchange the order of the concave $\log$ function  and expectation $\mathbb{E}\left[\cdot\right]$. There are conditions under which the lower bound $\mathcal{L}_\theta(x,a^\ast)$ is tight. For example, if all chain of thoughts $c$ in the support of $\pi_\theta(\cdot|x)$ induce the same probability of predicting the ground truth answer $\pi_\theta(a^\ast|x,c)$, i.e., $\pi_\theta(a^\ast|x,c)=\pi_\theta(a^\ast|x,c'),\forall c,c' \in \text{supp}\left(\pi_\theta(\cdot|x)\right)$. In practice when the optimization is approximate, such conditions are not likely to hold. As a result, there might be a gap between the lower bound and $\log \pi_\theta(a^\ast|x)$ and we will examine its empirical impact in practice.

The gap between the marginal log likelihood and the lower bound can be expressed as the KL divergence between $\pi_\theta$ and the posterior distribution \citep{blei2017variational}
\begin{align*}
    \log \pi_\theta(a^\ast|x) - \mathcal{L}_\theta(x,a^\ast) = \mathbb{KL}\left(\pi_\theta(\cdot|x),p^{\pi_\theta}(\cdot|x,a^\ast)\right) \geq 0,
\end{align*}
where $p^{\pi_\theta}(c|x,a^\ast)\coloneqq \frac{\pi_\theta(a^\ast|x,c)\pi_\theta(c|x)}{\sum_{c'}\pi_\theta(a^\ast|x,c')\pi_\theta(c'|x)}$ is the posterior, which defines a distribution over chain-of-thought given the \emph{prior} $\pi_\theta(c|x)$ and the \emph{likelihood} $\pi_\theta(a^*|x,c)$. For readers familiar with the probabilistic inference literature. The lower bound $\mathcal{L}_\theta(x,a^*)$ is closely related to the evidence lower bound  \citep{kingma2013,blei2017variational}, which we will elaborate more in Section~\ref{sec:connections}.

\subsection{Stochastic gradient estimate}

The lower bound permits stochastic gradient estimates. Concretely, given samples from the current policy $c\sim \pi_\theta(\cdot|x)$, we can construct an estimate of $\nabla_\theta\mathcal{L}_\theta(x,a^\ast)$ as 
\begin{align}
    \underbrace{\log 
 \pi_\theta(a^\ast|x,c) \nabla_\theta  \log \pi_\theta(c|x)}_{g_1} + \underbrace{\nabla_\theta \log \pi_\theta(a^\ast|x,c)}_{g_2}. \label{eq:gradient}
\end{align}
The gradient has two terms: $g_1$ is a REINFORCE gradient estimate with $\log 
 \pi_\theta(a^\ast|x,c)$ as the reward function for sampled chain-of-thought $c$ \citep{william1933}. The second gradient $g_2$ is reminiscent of a supervised learning loss that encourages the model to predict ground truth answer $a^\ast$ given sampled chain-of-thought $c$.

In practice, we can add a control variate to the REINFORCE gradient estimate to reduce variance. One option is to learn a prompt-answer dependent function  \citep{schulman2017}; another sample-based alternative is to generate $n$ i.i.d. chain-of-thoughts in parallel $c_i\sim \pi_\theta(\cdot|x)$, and construct leave-one-out control variates $v_i=\frac{1}{n-1}\sum_{j\neq i} \log \pi_\theta(a^\ast|x,c_j)$ \citep{mnih2016variational,kool2019buy,tang2025optimizing}. The overall gradient estimate is the average over $n$ samples:
\begin{align}
& \frac 1n \sum_{i=1}^n \Big[
 \left(\log \pi_\theta(a^\ast|x,c_i) - v_i\right) \nabla_\theta  \log \pi_\theta(c_i|x)\Big] \notag \\
&+  \frac 1n \sum_{i=1}^n \Big[ \nabla_\theta \log \pi_\theta(a^\ast|x,c_i)\Big].\label{eq:lower-bound-gradient}
\end{align}

Note the control variates $v_i$s do not introduce any bias to the gradient estimate since they are statistically independent from $\nabla_\theta \log \pi_\theta(c_i|x)$ and $\log \pi_\theta(a^*|x,c_i)$.

\paragraph{Connections to supervised fine-tuning} In the very special case where there is no chain-of-thought, the gradient estimate reduces to just the SFT part $
    \nabla_\theta \log \pi_\theta(a^\ast|x)$
which is effectively the supervised fine-tuning loss from prompt $x$ to answer $a^\ast$. Here, the key difference is that the loss $\pi_\theta(a^*|x,c_i)$ further conditions on the chain-of-thoughts $c_i$'s whose distribution changes over time and introduces more diversity to the optimization process. 

\section{Improving the objective via multi-sample Jensen's lower bound}\label{sec:algo-multi}

A loose lower bound induces a sizable discrepancy from the true objective of interest. A similarly simple yet tighter lower bound can be obtained with multiple samples \citep{burda2015importance}. Indeed, consider the $n$-sample lower bound
\begin{align}
\mathcal{L}_\theta^{(n)}(x,a^\ast)\coloneqq \mathbb{E}_{(c_i)_{i=1}^n\sim \pi_\theta(\cdot|x)} \! \left[  \log  \! \left(  \!
\frac{1}{n}\sum_{i=1}^n \pi_\theta(a^\ast|x,c_i) \! \right) \! \right] \! . \label{eq:simple-lower-bound-n}
\end{align}
Note that the $\log$ function is outside of the $n$-sample average to tighten the bound.
It is straightforward to verify that $\mathcal{L}_\theta^{(1)}(x,a^\ast)$ recovers the Jensen's lower bound as defined before in Eqn~\eqref{eq:simple-lower-bound}. As shown in \citet{burda2015importance}, the lower bound becomes tighter as $n$ increases $
   \mathcal{L}_\theta^{(n)}(x,a^\ast) \leq \mathcal{L}_\theta^{(n+1)}(x,a^\ast)$ for any $n\geq 0$. As $n\rightarrow\infty$, the bound approaches the marginal likelihood $\mathcal{L}_\theta^{(n)}(x,a^\ast)\rightarrow \log \pi_\theta(a^\ast|x)$, which is the ultimate objective of interest, under certain regularity conditions on $\pi_\theta$. 

To maximize the multi-sample lower bound $\mathcal{L}_\theta^{(n)}(x,a^*)$ with gradient ascent, we can construct a multi-sample stochastic gradient estimate as follows, 
\begin{align}
    &\underbrace{\sum_{i=1}^n \log \left(\frac{1}{n}\sum_{j=1}^n \pi_\theta(a^\ast|x,c_j)\right) \cdot \nabla_\theta \log \pi_\theta(c_i|x)}_{g_1^{(n)}} \notag \\
    &+ 
     \underbrace{\nabla_\theta  \log \frac{1}{n}\sum_{i=1}^n  \pi_\theta(a^\ast|x,c_i)}_{g_2^{(n)}}. \label{eq:multisample-gradient}
\end{align}
Empirically, the first term $g_1^{(n)}$ tends to have high variance as $n$ increases \citep{rainforth2018tighter}, since the objective $\log \frac{1}{n}\sum_{j=1}^n \pi_\theta(a^\ast|x,c_j)$ correlates updates to all $n$ samples. As a result, a key difference from the single-sample case is that the update is no longer an average over $n$ samples \citep{tang2025optimizing}. Akin to before, we can introduce the leave-one-out control variate without incurring any bias for variance reduction \citep{mnih2016variational,kool2019buy} with $\tilde{v}_i =\log \frac{1}{n-1}\sum_{j \neq i} \pi_\theta(a^\ast|x,c_j)$,
\begin{align*}
    \sum_{i=1}^n \left(\log \left(\frac{1}{n}\sum_{j=1}^n \pi_\theta(a^\ast|x,c_j)\right) - \tilde{v}_i\right) \cdot \nabla_\theta \log \pi_\theta(c_i|x).
\end{align*}

Note that the second term $g_2^{(n)}$, though can be estimated via random samples, is unlike a regular SFT loss. The key difference is that it is the log average of multiple probabilities, instead of the average of log probabilities as in the regular SFT loss. As $n\rightarrow\infty$, since $\log \frac{1}{n}\sum_{i=1}^n \pi_\theta(a^*|x,c_i)\rightarrow\log \pi_\theta(a^*|x)$, we see that conceptually $g_2^{(n)}$ can be understood as directly maximizing the marginal likelihood. In other words, the objective averages over multiple probabilities, which essentially marginalizes the chain-of-thought conditional distribution.

As we will show in Section~\ref{sec:exp-verifiable}, multi-sample lower bound generally improves the single-sample lower bound. This means that tightened lower bounds improve training objectives both in theory and in practice.

\section{Connections to algorithmic alternatives}\label{sec:connections}

The lower bound objectives bear close connections to a number of algorithmic alternatives, which we discuss below. See Algorithm~\ref{algo:online} for the pseudocode of the full algorithm, which we henceforth call JEPO.

\subsection{Evidence lower bound}

The evidence lower bounds (ELBO) \citep{blei2006variational,kingma2013,burda2015importance} controls for the tightness of the lower bound with an inference distribution $q_\phi(c|x,a^\ast)$ which defines a distribution over chain-of-thoughts. ELBO is usually written as follows
\begin{align}
    \mathcal{L}_{\theta,\phi}(x,a^\ast) = \mathbb{E}_{c}\left[\log \pi_\theta(a^\ast|x,c) - \log \frac{q_\phi(c|x,a^*)}{\pi_\theta(c|x)}\right],\label{eq:eval}
\end{align}
where the expectation is under $c\sim q_\phi(\cdot|x,a^\ast)$.
ELBO lower bounds the marginal log likelihood $\mathcal{L}_{\theta,\phi}(x,a^\ast)\leq \log \pi_\theta(a^\ast|x)$ and it is tight if and only if the inference distribution equals the posterior distribution $q_\phi(c|x,a^\ast)=p^{\pi_\theta}(c|x,a^\ast)$. Since ELBO is a function of both the policy parameter $\theta$ and inference distribution parameter $\phi$, given a chain-of-thought sample $c\sim q_\phi(\cdot|x,a^\ast)$, we can optimize both with stochastic gradient estimates:
\begin{align*}
    g_\theta &= \nabla_\theta \log \pi_\theta(a^\ast|x,c) + \nabla_\theta \log \pi_\theta(c|x),\\
    g_\phi &= \nabla_\phi \log q_\phi(c|x,a^\ast)\left(\log \pi_\theta(a^\ast|x,c) -\log \frac{q_\phi(c|x,a^*)}{\pi_\theta(c|x)}\right) \\
    &\ \ -\nabla_\phi \log q_\phi(c|x,a^*).
\end{align*}
Juxtaposing the form of the gradient here and the gradient to the Jensen's lower bound defined in Eqn~\eqref{eq:gradient}, we observe that the inference distribution gradient $g_\phi$ bears resemblance to the REINFORCE gradient; while the policy distribution gradient $g_\theta$ bears resemblance to the SFT gradient. In fact, we can show that under the special parameterization $q_\phi(c|x,a^\ast)\coloneqq\pi_\theta(c|x)$, the two gradients are exactly equivalent. More formally, we have the following.

\begin{lemma} (\textbf{Jensen's lower bound as a special case of ELBO}) \label{lemma:elbo-special-case}
    When $q_\phi(c|x,a^\ast)\coloneqq\pi_\theta(c|x)$, ELBO is equivalent to the Jensen's lower bound $\mathcal{L}_{\theta,\phi}(x,a^\ast)=\mathcal{L}_{\theta}(x,a^\ast)$ stochastic gradient estimates.
\end{lemma}
\begin{proof}
    When $q_\phi=\pi_\theta$, we have
    \begin{align*}
        g_\phi = \nabla_\theta \log \pi_\theta(c|x)\cdot \log \pi_\theta(a^\ast|x,c) - \nabla_\theta \log \pi_\theta(c|x)
    \end{align*}
    Adding this gradient to $g_\theta$, a simple manipulation shows that the aggregate gradient is equivalent to the gradient of the lower bound defined in Eqn~\eqref{eq:gradient}.
\end{proof}

With a parametric approximate posterior $q_\phi$, ELBO is more expressive than the Jensen's lower bound and allows for a tighter approximation to the marginal log likelihood. However, this also introduces additional complexity of having to learn the approximate posterior distribution. In our applications of interest, training a posterior model of a large size can be a major computational overhead. In practice, for example, \citet{hoffman2024training} approximates the posterior via a few steps of MCMC and avoids learning such a distribution. We take a different approach with a similar motivation: by tightening the lower bound with multiple samples, we also avoid the need for a parametric approximate posterior.

\subsection{Reinforcement learning}

We show that there is a close connection between the lower bound formulation and the expected return  maximization objective in RL \citep{sutton1998} for a single terminal reward. Concretely, we will see how the lower bound objectives are closely related to a \emph{conditional expectation trick} that produces a RL policy gradient estimate with lower variance. First, we show that (up to a log transform) RL optimizes for the same target as the lower bound objectives, given the indicator reward.

\begin{lemma} (\textbf{RL optimality is equivalent to maximum likelihood optimality}) \label{lemma:equivalence-optimality}
When $r(x,y)=\mathds{1}_{\{a=a^\ast\}}$, the optimal policy to the RL objective is equivalent to the optimal policy of the maximum likelihood objective Eqn~\eqref{eq:likelihood-objective}.
\end{lemma}
\begin{proof}
    The conclusion follows from the fact that $\mathbb{E}\left[\mathds{1}_{\{a=a^\ast\}}\right]=\pi_\theta(a^*|x)$. Hence the two objectives differ by a $\log$ operation and yield the same optimal solution.
\end{proof}

\begin{algorithm}[t]
    \begin{algorithmic}[1]
        \STATE \textbf{INPUT} policy $\pi_\theta$
        \WHILE { $t=0,1,2...$ }
        \STATE (i) For each sampled prompt $x$, collect $n$ generations $(y_i)_{i=1}^n$ and extract their corresponding chain-of-thoughts $(c_i)_{i=1}^n\sim \pi_\theta(\cdot|x)$.
        \STATE (ii) Evaluate $\pi_\theta(a^*|x,c_i)$ with a forward pass; calculate gradients $\nabla_\theta \log \pi_\theta(c_i), \nabla_\theta \log \pi_\theta(a^*|x,c_i)$ with backprop.
        \STATE (iii) Update $\theta$ with $n$-sample  gradient estimate Eqn~\eqref{eq:gradient} or its multi-sample variant Eqn~\eqref{eq:multisample-gradient}.
        \ENDWHILE
    \end{algorithmic}
        \caption{JEPO: Jensen's evidence lower bound policy optimization (for both single-sample and multi-sample lower bounds)}\label{algo:online}
\end{algorithm}

Assuming access to $n$ i.i.d. trajectories $(y_i)_{i=1}^n\sim \pi_\theta(\cdot|x)$, we start with the classic RL policy gradient with leave-one-out baseline (for example, RLOO \citep{rloo})
\begin{align}
\label{rloo-grad}
g_\text{vanilla-pg} = \frac{1}{n} \sum_{i=1}^n \nabla_\theta\log\pi_\theta(y_i|x) \cdot \left(\mathds{1}_{\{a_i=a^*\}} - w_i\right), 
\end{align}
where $w_i=\frac{1}{n - 1}\sum_{j\neq i} \mathds{1}_{\{a_j = a^*\}}$ is the leave-one-out baseline. Now, we present a new policy gradient estimate of the RL objective with guaranteed variance reduction, which is also feasible to implement with sample-based learning.

\begin{definition}[\textbf{A variance-reduced RL policy gradient estimate}]
Given $n$ trajectories $(y_i)_{i=1}^n$ from a single prompt $x$, we define $g_\text{var-reduced-pg}$ as
\begin{align}
\label{cond-rloo-grad}
\frac{1}{n}\sum_{i=1}^n  \nabla_\theta \log\pi_\theta(c_i)\cdot (\pi_\theta(a^*|c_i) - \tilde{w}_i) +  \nabla_\theta \pi_\theta(a^*|c_i),
\end{align}
where $\tilde{w}_i = \frac{1}{n - 1}\sum_{j\neq i} \pi_\theta(a^* | c_j)$ is the leave-one-out baseline akin to similar constructs in the lower bound case.
\end{definition}

We show that the variance-reduced policy gradient estimate is closely related to the classic gradient estimate via the conditional expectation trick.
\begin{lemma} \label{cond-grad-lemma} (\textbf{Conditional expectation})
Under the same assumption as Lemma \ref{lemma:equivalence-optimality} and denoting $a\sim \pi_\theta(\cdot|c)$ as the sampling process $a_i\sim \pi_\theta(\cdot|c_i)$, it holds that $g_\text{var-reduced-pg}$ is a conditional expectation of $g_\text{vanilla-pg}$
\begin{align}
\label{cond-grad-eqn}
g_\text{var-reduced-pg} = \mathbb{E}_{a\sim \pi_\theta(\cdot|c)} \left[g_\text{vanilla-pg} \;|\; (c_i)_{i=1}^n\right].
\end{align}
\end{lemma}
We note that without the leave-one-out baselines $\tilde{w}_i,\tilde{w}_i$, the conclusion Eqn~\eqref{cond-grad-eqn} is straightforward as both estimates Eqn~\eqref{cond-rloo-grad} and Eqn~\eqref{rloo-grad} become plain averages of i.i.d. terms. Now, using Lemma \ref{cond-grad-lemma}, we immediately see that the new gradient estimate yields smaller variance.
\begin{theorem} (\textbf{Variance reduction})
Under the same assumption as Lemma \ref{lemma:equivalence-optimality}, we have guaranteed variance reduction
\begin{align}
\label{var-red}
\mathbb{V}_{(y_i)_{i=1}^n\sim\pi_\theta(\cdot|x)} \left[g_\text{var-reduced-pg}\right] \le \mathbb{V}_{(y_i)_{i=1}^n\sim\pi_\theta(\cdot|x)} \left[g_\text{vanilla-pg}\right].
\end{align}
\end{theorem}

The proof is provided in Appendix~\ref{appendix:var}. Putting $g_\text{var-reduced-pg}$ from Eqn~\eqref{cond-rloo-grad} and the gradient estimate of the Jensen's lower bound (Eqn~\eqref{eq:gradient}) side-by-side, we identify intriguing similarities. Both gradient estimates employ two terms that update either the chain-of-thought component $\pi_\theta(\cdot|x)$ or the answer component $\pi_\theta(\cdot|x,c)$, with the only subtle difference being the extra log-transform needed for obtaining the Jensen lower bound. This alludes to the fact that the lower bound gradient has intrinsic  built-in variance reduction.

\subsection{Optimizing Jensen's lower bound with regularization is optimizing a special ELBO}

When optimzing the lower bound objectives, we also apply the KL regularization motivated from the regularized RL formulation (Eqn~\eqref{eq:rlhf}). Though this combination seems ad-hoc, we will see that optimizing such an hybrid objective is in fact equivalent to maximizing a special ELBO.

Incorporating the regularization into the lower bound formulation, we have an aggregate objective
\begin{align}
     \mathcal{L}_{\theta}(x,a^*) - \beta\mathbb{KL}(\pi_\theta,\pi_\text{ref}).\label{eq:regularized-elbo}
\end{align}
If we refine the regularization a little more: instead of the generation level regularization, we apply regularization at the chain-of-thought: $\mathbb{KL}_c\left(\pi_\theta,\pi_\text{ref}\right)\coloneqq\mathbb{E}_{c\sim\pi_\theta(\cdot|x)}\left[\log\frac{\pi_\theta(c|x)}{\pi_\text{ref}(c|x)}\right]$, then the resulting aggregate objective can be interpreted in a more coherent way, as an ELBO to a concrete generative process.
\begin{lemma}\label{lemma:regularized}
    (\textbf{Regularized lower bound as an ELBO to a special generative process}) Assume a generative process $c\sim \pi_\text{ref}(\cdot|x),a\sim \pi_\theta(\cdot|x,c)$ that defines a marginal distribution $p_{\pi_\theta,\pi_\text{ref}}(a|x)\coloneqq \sum_c \pi_\text{ref}(c|x)\pi_\theta(a^*|x,c)$. Then the regularized objective $\mathcal{L}_{\theta}(x,a^*) - \mathbb{KL}_c(\pi_\theta,\pi_\text{ref})$ is a lower bound to the log likelihood $\log p_{\pi_\theta,\pi_\text{ref}}(a|x)$.
\end{lemma}
\begin{proof}
    Applying the same derivation as the regular ELBO, the log likelihood $\log p_{\pi_\theta,\pi_\text{ref}}(a|x)$ is lower bounded as
    \begin{align*}
         &\geq \max_\phi\  \mathbb{E}_{c\sim q_\phi(\cdot|x,a^\ast)}\left[\log \pi_\theta(a^\ast|x,c) - \log \frac{q_\phi(c|x,a^*)}{\pi_\text{ref}(c|x)} \right]  \\
        &\geq_{(a)} \mathbb{E}_{c\sim \pi_\theta(\cdot|x)}\left[\log \pi_\theta(a^\ast|x,c) - \log \frac{\pi_\theta(c|x)}{\pi_\text{ref}(c|x)} \right] \\
        &= \mathcal{L}_{\theta}(x,a^*) - \mathbb{KL}_c(\pi_\theta,\pi_\text{ref}),
    \end{align*}
where inequality (a) is due to choosing $q_\phi=\pi_\theta$ and the last equality is by definition. Hence the proof is complete.
\end{proof}

Note that the aggregate objective Eqn~\eqref{eq:regularized-elbo} can also be optimized via stochastic gradient ascent with standard estimates. We just need to add an additional term associated with the KL regularization, to the original gradient estimate to $\mathcal{L}_\theta(x,a^*)$ defined in Eqn~\eqref{eq:gradient}. An example of such a gradient estimate is the following
\begin{align*}
    \log \frac{\pi_\theta(c|x)}{\pi_\text{ref}(c|x)}\nabla_\theta \log \pi_\theta(c|x),c\sim \pi_\theta(\cdot|x).
\end{align*}
Though our lower bound interpretation (Lemma~\ref{lemma:regularized}) is under a regularization only on the chain-of-thought, in practice, we still apply the full generation level regularization following common practice \citep{christiano2017deep,ziegler2019fine,ouyang2022training}.

\subsection{Practical trade-offs comparing JEPO vs. RL}

As discussed earlier, JEPO does not require an external verifiable reward, instead, it can be understood as applying the indicator reward $r(x,y)=\mathds{1}_{\{a=a^\ast\}}$. In practice, this can be instantiated as a strict string match \verb|float(answer == gt_answer)|. However, such a reward function will likely induce false negatives, as semantically equivalent generations might be vastly different strings. In practice, a more lenient match is typically applied to  remove more false negatives. For example, for math problems \citep{hendrycks2021measuring,yue2024harp}, usually programmtic checks are implemented to check for equivalence of two short expressions, such that e.g., \verb|pi|  and \verb|3.1415926| might be considered equivalent.

More formally, consider a general reward function $r(x,y)=\text{match}(a,a^*)$ calculated as a binary match between $a$ and $a^*$. We can rewrite the RL objective as $\mathbb{E}[\text{match}(a,a^*)]$. In order to adapt JEPO to the new RL ojbective, we need to work  with the equivalent set $\mathcal{A}\coloneqq \{a|\text{match}(a,a^*)=1\}$ as well as  quantities such as the set probability  $\pi_\theta(\mathcal{A}|x,c)\coloneqq\sum_{a\in \mathcal{A}} \pi_\theta(a|x,c)$. Note that this probability 
 reduces to $\pi_\theta(a^*|x,c)$ in case we use exact match. In general, computing such probabilities is expensive since we need to enumerate all $a\in\mathcal{A}$ if inverting the match function is feasible at all. As a result, it is challenging to adapt JEPO to generic match function or reward function.

In summary, when a good verifiable reward is available (Sympy vs. string for certain math datasets, see semi-verifiable experiments in Section~\ref{sec:exp-semi}), online RL is at an advantage. There are also cases where good rewards are not readily available and JEPO is a reasonable algorithm. An example is where the ground truth answer takes a rather long form, e.g., see unverifiable experiments in Section~\ref{sec:exp-unverifiable}.

\section{Implementation details}\label{sec:implement}

We explain the implementation details of the JEPO algorithm in this section. We highlight a few key technical elements for the practical implementation, which we have found to be important in getting the best performance.

\paragraph{Formatting penalty}
We find it useful to have an additional RL loss with the reward as $r_\text{format}(x,y)=-p$ if $y$ does not follow the formatting requirement (that the identifier phrase \emph{the final answer is} is in $y$) and zero otherwise. We find that this generally helps stabilize the training process. This is especially useful for small models (8B), where, under temperature sampling, it can often not follow instructions strictly. For large models (70B), we also found that its formatting might be inconsistent after multi-epoch training. We find a value of $p=10$ suffices while smaller values tend to make the training less stable due to weaker penalties.

\paragraph{Per-sequence log probs} For the \emph{log-ave-exp} operation that defines the lower bound in Eqn~\eqref{eq:simple-lower-bound}, it is important to apply the per-sequence log probs without averaging over the sequence length. Concretely, the bound is calculated as
\begin{align*}
     \log \left(\frac{1}{n}\sum_{j=1}^n \sum_{t< |a^\ast|} \pi_\theta(a_t^\ast|x,c_j,a_{<t}^\ast)\right),
\end{align*}
where $|a^*|$ denotes the sequence length of the ground truth $a^*$. It is important \emph{not} to average the sequence level log probs $\log \sum_{t< |a^\ast|} \pi_\theta(a_t^\ast|x,c_j,a_{<t}^\ast)$ with a length factor of $1/|a^\ast|$ as suggested in other contexts \citep{grinsztajn2024averaging,shao2024deepseekmath}. The reason is that the algorithm seeks to make $a^*$ more likely, and the sequence level log probs comply with this goal. The length normalization can modify the objective landscape significantly especially when $|a^*|$ is large. For example, JEPO algorithm does not work well on the proof data when length normalization is applied.

\paragraph{Advantage normalization} Both the baseline RL and JEPO apply advantage post-processing, following common practice in prior work \citep{schulman2017,baselines}. For example, in the multi-sample JEPO, the raw advantage for the $i$-th generation is 
\begin{align*}
    A_i = \log \left(\frac{1}{n}\sum_{j=1}^n \pi_\theta(a^\ast|x,c_j)\right) - \tilde{v}_i,
\end{align*}
where $\tilde{v}_i$ is the control variate. A further normalization is applied to the advantage $\tilde{A}_i = \text{clip}(A_i / \text{std}\left(A\right), -1, 1)$ such that the outcome $\tilde{A}_i$ is applied in the actual update. Advantage normalization is especially important for JEPO because its raw advantage takes a wider range of values, compared to RL with binary reward.

\paragraph{Weighted supervised learning loss} We also introduce a weighting coefficient for the supervised loss $\beta_\text{sup}\geq0$, which we found useful for ablations. We observe that small values $0\sim 10^{-2}$ tends to work for short-answer applications (e.g., MATH) while a large value $\sim 1$ is important for semi long-form data (e.g., numina and numina-proof), in order to place more weight on the supervised learning loss.

\paragraph{KL-regularization} In early investigations, we found it useful to have a KL regularization at a very small coefficient $\beta=10^{-3}$. The regularization helps prevent formatting collapse, and also prevents the policy from drifting too much in case the updates are noisy \citep{ziegler2019fine,ouyang2022training}. 
\begin{figure*}[t]
\centering
\includegraphics[width=\textwidth]{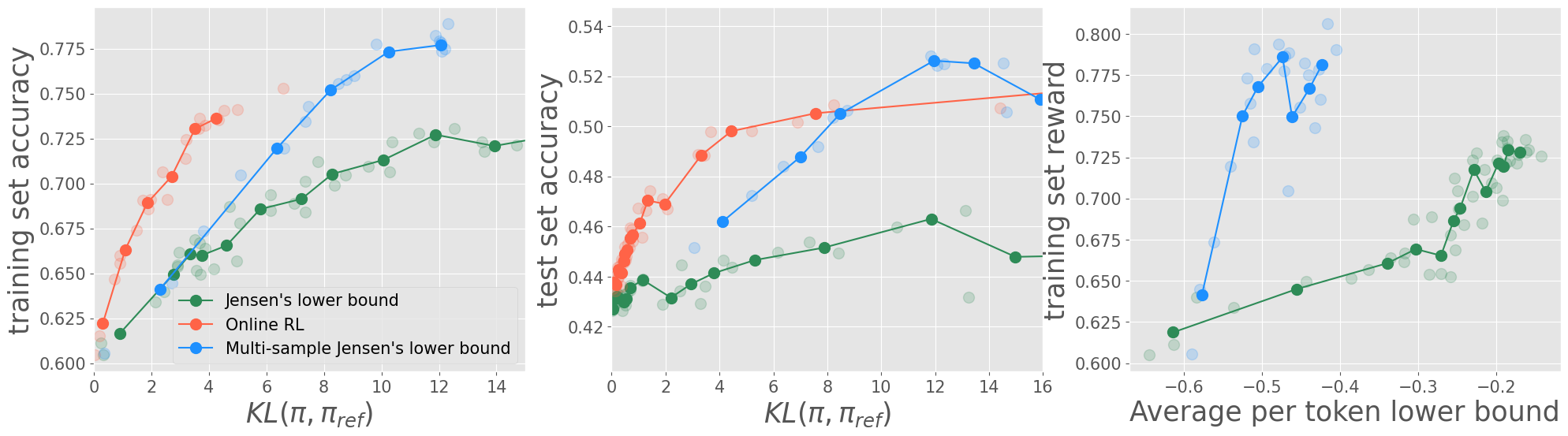}
\caption{\small{Verifiable data experiments with MATH. We compare three baselines: online RL with access to the oracle Sympy-based reward and JEPO. In the left plot, we monitor the reward on the training dataset. Online RL obtains the best training time trade-off, followed by multi-sample lower bound and the single-sample lower bound; In the middle plot, we monitor the evaluation on a test set during training. Multi-sample lower bound and online RL obtains similar performance; In the right plot, we graph training reward against the lower bound objectives, averaged over training tokens. The two signals bear positive correlations overall and multi-sample lower bound yields better correlations.}}
\label{figure:onpolicy-steps}
\end{figure*}

Put together, given $n$ samples, the JEPO update is
\begin{align*}
    &\frac{1}{n} \sum_{i=1}^n \left(\left(\tilde{A}_i + \tilde{A}_i^{\text{(ref)}}\right) \nabla_\theta \log \pi_\theta(c_i|x)   \right) \\ &+ \beta_\text{sup} \nabla_\theta  \log \left(\frac{1}{n}\sum_{i=1}^n  \pi_\theta(a^\ast|x,c_i)\right) \\ &- \beta \nabla_\theta \mathbb{KL}\left(\pi_\theta(\cdot|x), \pi_\text{ref}(\cdot|x)\right),
\end{align*}
where $\tilde{A}_i^{\text{ref}}$ is the normalized advantage for the formatting penalty.
The advantage normalization and weighting coefficient $\beta_\text{sup}$ make it such that the ultimate update optimizes for a weighted lower bound with resemblance to to $\beta$-VAE \citep{higgins2017beta}. We encourage readers to reference against the online RL implementation in Appendix~\ref{appendix:hypers} to understand its similarities with the JEPO algorithm.

The JEPO loss is only applied to generations with correct format, otherwise, the loss is masked out. The formatting advantage update is applied to all generations. Also, we find that the sequence level normalization with a factor of $1/\left(|c_i|+|a^*|\right)$ or $1/|c_i|$ does not make a significant difference in performance \citep{shao2024deepseekmath,liu2025understanding}.

\section{Related work}

\paragraph{Training with unverifiable data} A natural way to generalize RL training to unverifiable data is to make use of LLM feedback, e.g., \emph{LLM-as-judge} uses LLM to assess the quality of the generated response \citep{lee2023rlaif,guo2024direct,yuan2025selfrewardinglanguagemodels}. However, despite its conceptual simplicity, LLM-as-judge might not produce reliable assessment for domain-specific or long-form data \citep{lightman2023let,petrov2025proof}. When optimizing against judge scores, it is also more likely to over-optimize \citep{gao2023scaling}. As a result, in this work we apply LLM-as-judge only for short-form evaluations and not for training.

Closely related to our work is the concurrent VR-CLI (verifiable reward with completion likelihood improvement) \citep{gurung2025learning} where they apply log probs of golden generations as reward. Using our terminology, their approach resembles the first part of the gradient in Eqn~\eqref{eq:lower-bound-gradient} of the Jensen's evidence lower bound. Without a SFT-like component, their update does not optimize for the marginal likelihood only partially. JEPO also applies the multi-sample technique to tighten the lower bound, achieving better empirical performance, which we will demonstrate in Section~\ref{sec:exp-verifiable}.

\paragraph{Likelihood-based scoring}
Prior work showcased the utility of Likelihood-based scoring in filtering of chain-of-thought \citep{zelikman2024quiet,ruan2025reasoning}. The algorithms mostly proceed in an iterative fashion akin to expectation-maximization \citep{moon1996expectation}, which in theory can also maximize the evidence of the desirable final answers. Complementary to such work, since we extend the training process to fully online RL settings, we forgo the need of variational posteriors which allows for training on unverifiable data at scale. We also demonstrate performance gains beyond short-form answers, which were the main focus of prior work.

\paragraph{Chain-of-thought as latent variable modeling} The idea of casting optimizing chain-of-thought as latent variable modeling is not new. Previously, \citet{hoffman2024training} proposed an algorithm motivated by maximizing ELBO to tackle reasoning problems. Such an algorithm also draws close connections to prior work \citep{zelikman2022star,gulcehre2023reinforced,singh2023beyond,yuan2023scaling} all of which resemble a hybrid offline-online RL training loop, where they alternate between sampling and filtering via a reward. They also  have an interpretation as EM algorithmic variants \citep{moon1996expectation}.

Despite the appeal of the full ELBO formulation, it is rarely implemented in practice due to the requirement of learning the posterior distribution. Indeed, despite the formulation of \citet{hoffman2024training} they ended up approximating the posterior with MCMC, which effectively made use of an explicit reward to filter samples. This also marks a key difference from our work - we do not apply any explicit reward scoring throughout our algorithmic design and practical implementation. In addition, \citet{hu2024unveiling} has proposed a more systemic hierarchical latent variable modeling view of chain-of-thought. Similar to our motive, \citet{sordoni2024joint} optimized an ELBO inspired objective for prompt selection, where they did not resort to an external reward. 

\paragraph{Evidence lower bound and RL} The connections between evidence lower bound and RL has been extensively studied in both the variational inference \citep{ranganath2016hierarchical,blei2017variational} and RL community \citep{levine2018reinforcement,o2020making}. In the RL literature, much of the variational inference view has been used to better interpret and improve existing algorithms with much focus on the goal-conditional problems, where a single reward is assigned at the end of a trajectory. Such a setting is quite akin to the RLHF case, where a sequence terminates with a single reward \citep{andrychowicz2017hindsight,eysenbach2020rewriting,tang2021hindsight}. Our formulation also naturally incorporates the tighter multi-sample lower bound \citep{burda2015importance,rainforth2018tighter} as special cases, which has seen little adoption in prior RL literature.

\section{Experiments with verifiable data}\label{sec:exp-verifiable}

We start by comparing JEPO against RL baselines on verifiable data. We focus on the mathematical reasoning dataset MATH \citep{hendrycks2021measuring} where the prompt $x$ asks the model a mathematical question with a short form answer $a^*$. We study the two algorithmic variants proposed in this work: the JEPO defined through the gradient estimate in Eqn~\eqref{eq:gradient} as well as its multi-sample variant Eqn~\eqref{eq:multisample-gradient}. As a strong baseline, we consider 
the online policy gradient RL algorithm which applies Sympy \citep{10.7717/peerj-cs.103} to automatically match the answers. The RL algorithm applies leave-one-out for variance reduction, as is commonly practiced \citep{rloo,shao2024deepseekmath}. Our main experiments are based on the 8B and 70B model from the Llama 3.1 model family \citep{dubey2024llama}. All algorithmic variants apply identical hyper-parameters such as learning rate, and that they all apply $n=4$ samples for gradient estimations, which we detail in Appendix~\ref{appendix:hypers}.

The RL baseline is at an advantage in this setting, since the reward is of high quality and is itself being used as evaluation signals too \citep{yue2024harp}. We do not compare with other baselines developed in prior work (e.g., \citep{hoffman2024training}) as they can be interpreted as variants of online RL algorithms with relatively minor algorithmic differences.

\begin{figure}[t]
\centering
\includegraphics[width=\linewidth]{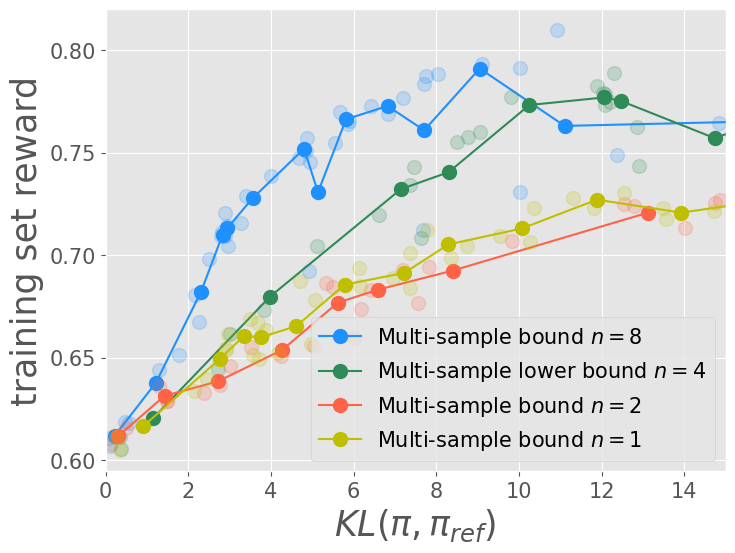}
\caption{\small{Ablation of number of samples $n$ for multi-sample lower bounds. As we increase the number of samples, the multi-sample lower bound seems to further improve the training-time efficiency. This corroborates the theoretical insight that as $n$ increases, the multi-sample lower bound objectives become tighter.}}
\label{figure:iwae-ablations}
\end{figure}

\subsection{Comparison on MATH}

During RL training, we use a reward of $r=1$ when there is an answer match and $r=0$ otherwise. Note that JEPO does not require access to such a reward, but we monitor the reward scores during training. Figure~\ref{figure:onpolicy-steps} left plot shows the training performance of all baselines. For the x-axis, we use the KL divergence $\mathbb{KL}(\pi_\theta,\pi_\text{ref})$ calculated over the training set. Following the practice in \citep{gao2023scaling}, we adopt the KL divergence as a certain measure of the optimization budget that the algorithm has consumed. Note that here all experiments are run with the same regularization coefficient $\beta=10^{-3}$ since it achieves a good trade-off for all algorithmic variants over all. 

\paragraph{Training performance} Figure~\ref{figure:onpolicy-steps} left plot shows that online RL achieves a good KL-performance trade-off on the training set. This is probably not a big surprise since online RL optimizes for the very same objective that we monitor here. In the meantime, JEPO enjoys reasonable performance: as the policy deviates from the reference policy, the reward performance improves despite not explicitly training for it (in theory JEPO optimizes for a hard string match rather than Sympy match). (2) the multi-sample JEPO obtains noticeably better performance than the one-sample lower bound baseline, despite using the same $n=4$ generations per update. We will ablate on the impact of parameter $n$ on the performance.

\paragraph{Evaluation} Figure~\ref{figure:onpolicy-steps} middle plot shows the evaluation performance on an held-out test set. We note that the reward on the training set is higher than the test set, because the model has been SFT'ed on on the training prompts. For evaluation, observe that the multi-sample lower bound method obtains similar performance as online RL, despite being outperformed during training. We conjecture that this is because online RL tends to overfit the training prompts more significantly, producing a high training reward that does not transfer as well to the evaluation time. This shows that even without training on the reward signal explicitly, JEPO can obtain a similar evaluation performance as online RL.

\paragraph{Statistical correlation between objectives} Figure~\ref{figure:onpolicy-steps} right plot graphs the training time reward against the lower bound objectives. If we consider the training reward as a ground truth objective to optimize for, we see that the multi-sample lower bound displays a stronger correlations between the surrogate objective and the ground truth. This corroborates with the observation that multi-sample lower bound tends to lead to better performance, compared to single-sample lower bound.

\subsection{Ablation study}

We now provide ablation results on a few important dimensions of the algorithm.

\begin{figure}[t]
\centering
\includegraphics[width=\linewidth]{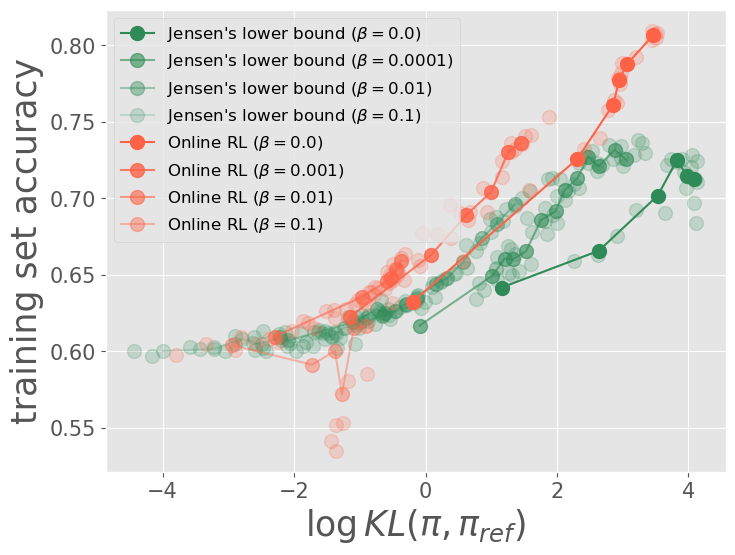}
\caption{\small{Ablation of regularization coefficient $\beta$. As $\beta$ increases, all algorithmic variants seem to obtain better efficiency in the training performance-KL divergence trade-off. However, strong regularization also prevents the policy from deviating much from the reference policy, preventing bigger training improvements.}}
\label{figure:kl-ablations}
\end{figure}

\paragraph{Multi-sample ablation on sample size $n$} We ablate on the number of sample $n$ used for constructing per gradient update. In theory, as $n$ increases, the multi-sample lower bound becomes tighter and asymptotically approaches the marginal likelihood objective (which is equivalent to the RL objective). We vary the sample size $n\in\{1,2,4,8\}$ and compare the performance. Figure~\ref{figure:iwae-ablations} shows that as $n$ increases, the algorithm becomes more KL-efficient: with a fixed budget on KL, the model obtains better performance. Intriguingly, we also observe a training performance akin to reward over-optimization \citep{gao2023scaling} - as the optimization progresses, the training reward drops slightly (for blue curve). We can interpret this as the result of the fact that JEPO does not optimize for the same indicator matching function as the reward we monitor.

\paragraph{Regularization ablation} We investigate the impact of the regularization coefficient $\beta\in\{0,10^{-3},10^{-2},10^{-1}\}$. Figure~\ref{figure:kl-ablations} shows the training performance of the single-sample lower bound vs. online RL. One observation is that as $\beta$ increases, the trade-off efficiency for both algorithms improves - however, in general the algorithm also makes less deviation from the reference policy, hence leading to less improvement for a fixed training steps.

\begin{figure}[t]
\centering
\includegraphics[width=\linewidth]{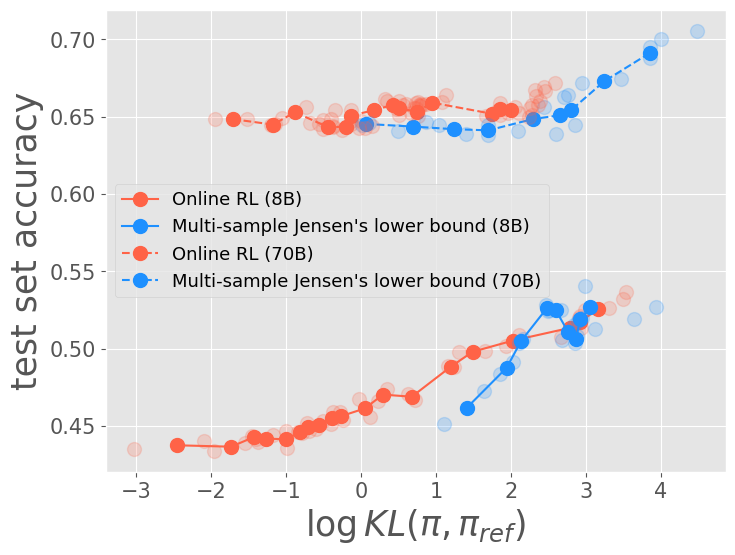}
\caption{\small{Ablation of model size (8B vs. 70B). We find that the multi-sample JEPO is fairly competitive against the online RL algorithm in the 70B scale. Both algorithm traces out a similar KL-performance trade-off, with multi-sample JEPO obtaining a slightly better performance given a similar compute budget as online RL.}}
\label{figure:size-ablations}
\end{figure}

\begin{figure*}[t]
\centering
\includegraphics[width=\textwidth]{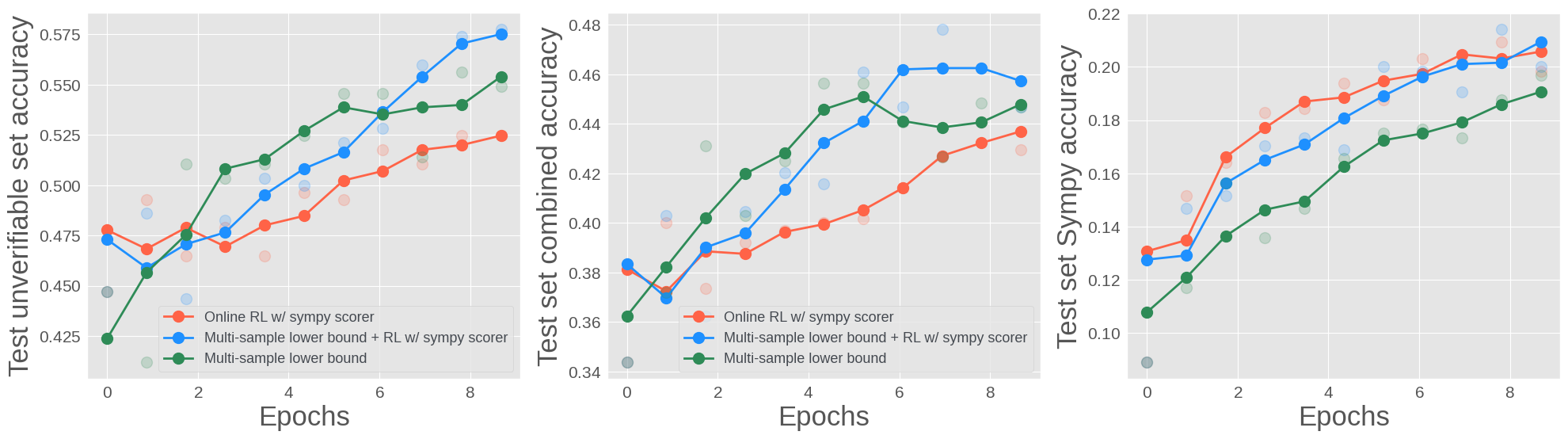}
\caption{\small{Evaluation comparison of training 70B models on semi-verifiable Numina dataset. We show evaluation results during the course of training. Left plot shows the combined accuracy on the unverifiable subset (about 40\%) of the test set; middle plot shows the combined accuracy on the full test set; right plot shows the Sympy score on the full test set. While JEPO progresses more slowly on the Sympy scores compared to online RL, it gains on the combined accuracy; the combined algorithm seems to achieve the best of both worlds.}}
\label{figure:numina-comparison}
\end{figure*}

\paragraph{Scaling up model size} Since the multi-sample JEPO appears more competitive, we compare it against the online RL in the 70B case. Figure~\ref{figure:size-ablations} shows that the JEPO obtains competitive performance against online RL in terms of the KL-performance trade-off. With roughly the same amount of compute budget, we find that the JEPO seems to drift further from the reference policy, hence extending the trade-off curve to a performance of $70\%$ test set accuracy, which outperforms online RL modestly.

\paragraph{Supervised loss} We find that a low value of $\beta_\text{sup}$ generally works better for the JEPO algorithms. The speculation is that when $\beta_\text{sup}$ is large, the supervised loss encourages the policy to place weights on the ground truth $a^*$ despite that the chain-of-thought $c$ has low quality. This leads to overfitting the training set, in a more severe way than online RL. This is because the JEPO supervised learning loss incentivizes the model to directly memorize $a^*$ given any context $(x,c)$, by maximizing the likelihood $\log \pi_\theta(a^*|x,c)$.

Interestingly, we will show that with long-form data, large values of $\beta_\text{sup}$ work generally better, since the risk of overfitting is less severe.

 \section{Experiments with semi-verifiable data}
\label{sec:exp-semi}

We now consider semi-verifiable data where a good proportion of the dataset contains answers which are not easily verifiable. We focus on a post-processed Numina dataset \citep{numina_math_datasets} where prompts are mathematical questions and ground truth answers are partly verifiable. For instance, one example of the ground truth answer is the whole expression: {\tt $\forall x \in \mathbb{R}, x^2 + (a-1)x + 1 \geq 0$}. Given a model generated answer, it is hard to verify whether it is equivalent to the above expression without case-specific parsing; often time, such parsing results in false negatives. See Appendix~\ref{appendix:hypers} for details on how we post-process the dataset and data splits.

\paragraph{RL baseline and reward} For the RL baseline, we apply the Sympy reward as introduced in the previous section. Because the dataset contains answers which are hard to verify, the reward is effectively only applicable to a subset of the data. The default training set contains about 22k examples. We estimated at least 40\% of such examples cannot be verified by the automatic scorer. We consider online RL with such reward as a baseline, as it has access to a highly specialized verifiable reward but only applicable to a subset of the data.

\paragraph{Combining JEPO and RL baseline} We also compare with an algorithm that combines the loss function of JEPO and RL baseline with the Sympy reward. When we sample $n$ generations from a single prompt, and if none of the generation obtains a positive score (note this does not mean that the example is necessarily unverifiable), we apply the JEPO loss; otherwise, we apply the baseline RL loss. This allows for a dynamic combination of two losses, and still leverages the whole dataset.

\paragraph{Evaluation} As with the trainging set, the held-out test set contains both verifiable and unverifiable examples, which we evaluate in two ways: (1) Sympy reward $r_\text{sympy}(a,a^*)\in\{0,1\}$, which generally underestimates the true accuracy when ground truth is semi-verifiable; (2) Sympy combined with LLM-as-judge $r_\text{combined}(a,a^*)$, which combines two sources of scores 
\begin{align*}
    r_\text{combined}(a,a^*) \coloneqq r_\text{sympy}(a,a^*) + r_\text{llm}(a,a^*)\mathds{1}_{\{r_\text{sympy}(a,a^*)=0\}}.
\end{align*}
The LLM-as-judge score $r_\text{llm}(a,a^*)$ is also binary: it is based on a 5-time majority voted decision of a prompted 70B instruction-tuned model \citep{dubey2024llama}. Though imperfect, we observe that LLM-as-judge reasonably mitigates some false negatives caused by rigid Sympy scoring.
Importantly, we reiterate that we do not train on such combined scores - they are used for evaluations only.

\subsection{Comparison on Numina}

Unless otherwise stated, we will experiment with the multi-sample algorithm given its performance gains in Section~\ref{sec:exp-verifiable}. Below, Figure~\ref{figure:numina-comparison} shows the evaluation performance comparing the RL baseline, JEPO and their combined algorithm. Since the Numina dataset is more challenging, we experiment throughout with 70B models.

\paragraph{Sympy scoring evaluation}
Figure~\ref{figure:numina-comparison} 
 right plot shows the evaluation accuracy using the Sympy score. Overall, all algorithms make steady progress as the training progresses. However, since online RL baseline trains with the same reward signal, it achieves slight acceleration compared to JEPO. The combined algorithm achieves a similar rate of progress with the Sympy scores on the test set.

Due to the abundance of symbolic expressions as ground truth answers in the Numina dataset, here the Sympy reward is a much more specialized scoring method than e.g., string match compared to the MATH case. This partly explains why the online RL baseline is quite competitive, as it also trains on the very same signal. Note this experimental result corroborates the trade-off discussion in Section~\ref{sec:connections}.

\paragraph{Combined scoring evaluation}
Figure~\ref{figure:numina-comparison} left plot and middle plot shows the combined accuracy which alleviates some false negatives due to the Sympy scoring. As seen from the overall metrics, the accuracy increases by about 25\% compared to the Sympy scores. The left plot shows the performance on the unverifiable test subset (40\% of the test set) while the middle plot shows the full set. We observe that both JEPO and the combined algorithm achieves faster rate of progress and asymptotes to slightly better performance than the online RL baseline with this combined metric, especially on the unverifiable subset. Interestingly, note that by training on verifiable rewards, online RL can also make progress on the unverifiable test set.

Though the Sympy scoring is quite specialized, it is only applicable to a subset of the full Numina training set. Meanwhile, JEPO can leverage the full dataset, despite with less specialized signals. The combined algorithm seems to achieve the best of both worlds.

\begin{figure}[t]
\centering
\includegraphics[width=\linewidth]{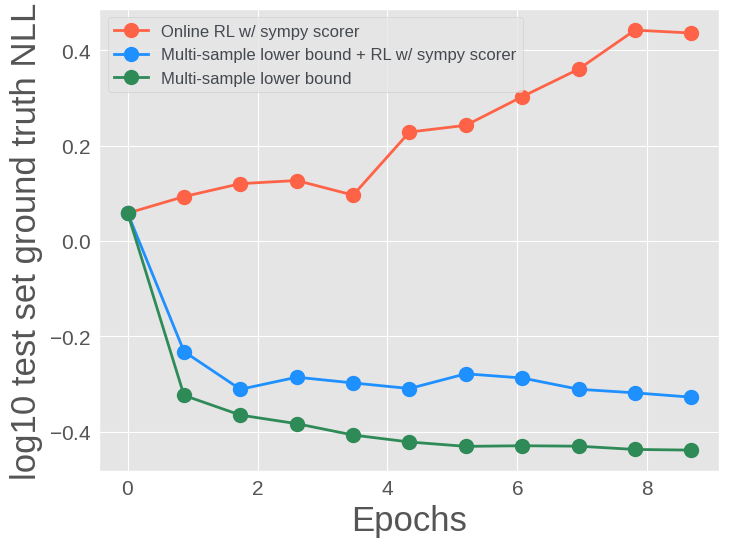}
\caption{\small{Test set proxy NLL evaluation for training on the numina dataset. We evaluate the proxy NLL of the trained models on the numina test set, approximated with $n=4$ samples lower bound defined in Eqn~\eqref{eq:eval}. Both JEPO and the combined algorithm sees improvement in the NLL (lower the better), while online RL does not improve on test set NLL. This hints at different solutions found by the online RL and JEPO algorithms, despite similar improvement trend in the sampling based evaluations.}}
\label{figure:numina-nll}
\end{figure}

\subsection{Ablation study}

We discuss a few additional ablations on the Numina dataset.

\paragraph{Test set negative likelihood: lower the better} We further evaluate the proxy negative log likelihood (NLL) that the trained model produces on test set, computed via the $n$-sample lower bound
\begin{align*}
    \text{proxy-NLL}(\pi_\theta) = -\mathbb{E}\left[\log \left(\frac{1}{n}\sum_{i=1}^n \log \pi_\theta(a^*|x,c_i)\right)\right]
\end{align*}
where the expectation is under $(c_i)_{i=1}^n\sim \pi_\theta(\cdot|x), (x,a^*)\sim\mathcal{D}_\text{test}$, following common practice \citep{burda2015importance,ruan2025reasoning}. Figure~\ref{figure:numina-nll} shows such proxy NLL during training, where we see a different pattern for the online RL baseline and JEPO. For JEPO, the proxy NLL decreases over time. We expect such a result because JEPO optimizes for the same objective on the training set, and before overfitting, we expect improvement on the test set. 

Meanwhile, maybe surprisingly, online RL does not make progress on the test set NLL. The combined algorithm is in between the two extremes. There are good reasons for online RL not to make progress on test set NLL. Particularly, for each ground truth answer in the dataset $a^*$, the Sympy scorer defines a sizable collection of correct answers $\mathcal{A}=\{a: r_\text{sympy}(a,a^*)=1\}$ whose aggregate probability $\pi_\theta(\mathcal{A}|x)$ increases under online RL (evidenced by test set accuracy improvement in Figure~\ref{figure:numina-comparison} right plot). In other words, online RL might not improve the proxy NLL of a particular $a^*$ (defined through the dataset) inside $\mathcal{A}$.

The above observation implies that the policy found by online RL and JEPO can produce different answers to the same question. It is also suggestive of how reward-based RL post-training can change the calibration behavior of likelihood-based models \citep{achiam2023gpt}.

\paragraph{Comparison with SFT baseline on golden chain-of-thought} To assess another option to improve semi-verifiable performance, we carry out another comparison against a SFT baseline, which trains on the golden chain-of-thought found in the source dataset \citep{numina_math_datasets}. We observe performance improvements across evaluation metrics as well, though generally under-performing RL. See Appendix~\ref{appendix:semi-verifiable} for full results.

\section{Experiments with unverifiable data}
\label{sec:exp-unverifiable}

Finally, we experiment with unverifiable data, where the full dataset has long-form ground truth and whose equivalence against another solution cannot be easily verified with hard-coded programs. We consider a post-processed Numina-proof, extracted from the original Numina dataset. The proof often contains multiple sentences or paragraphs, without a final short-form answer as in MATH or the verifiable subset of Numina.

\begin{figure}[t]
\centering
\includegraphics[width=\linewidth]{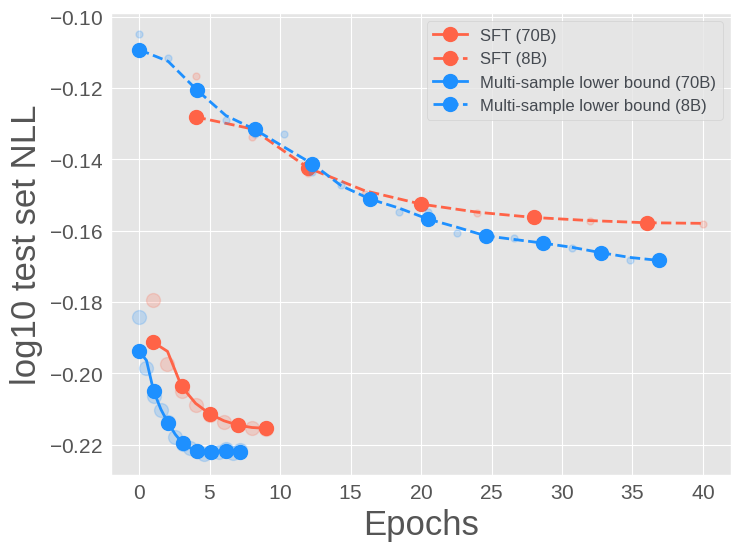}
\caption{\small{Test set proxy NLL evaluation for training on the unverifiable Numina-proof dataset. For JEPO outperforms the SFT baseline at the same data budget (measured in epochs) and achieves asymptotically better test performance.}}
\label{figure:numina-nll-proof}
\end{figure}

\paragraph{Baselines} Since the ground truth is long-form and cannot be verified easily, we do not have a RL baseline with verifiable reward. Instead, SFT on the raw dataset $(x,a^*)$ is a reasonable baseline. Through a few ablations, we also compare with methods akin to VR-CLI \citep{gurung2025learning}, which corresponds to the REINFORCE part of the single-sample lower bound gradient in Eqn~\eqref{eq:lower-bound-gradient}.

\paragraph{Evaluation} We evaluate NLL on the test set, akin to the ablations in Section~\ref{sec:exp-semi}. We do not carry out sampling based evaluations as long-form answers are hard to assess even for frontier models \citep{petrov2025proof}.

\subsection{Comparison on Numina-proof}

As main experiments, we compare JEPO with SFT. Note that we always started with instruction-tuned models \citep{dubey2024llama} and the SFT baseline can be understood as a continued SFT. We show the curve after an initial transient phase where the test set NLL drops significantly for all runs, which can be attributed to that the modes learn to format answer correctly. 

Figure~\ref{figure:numina-nll-proof} shows the test set NLL comparison between SFT and JEPO, with both 8B and 70B models. At both scales, JEPO outperforms SFT with test set NLL at the same training data epoch. Meanwhile, JEPO also achieves asymptotically better NLL than SFT.

\subsection{Ablation study}

\begin{figure}[t]
\centering
\includegraphics[width=\linewidth]{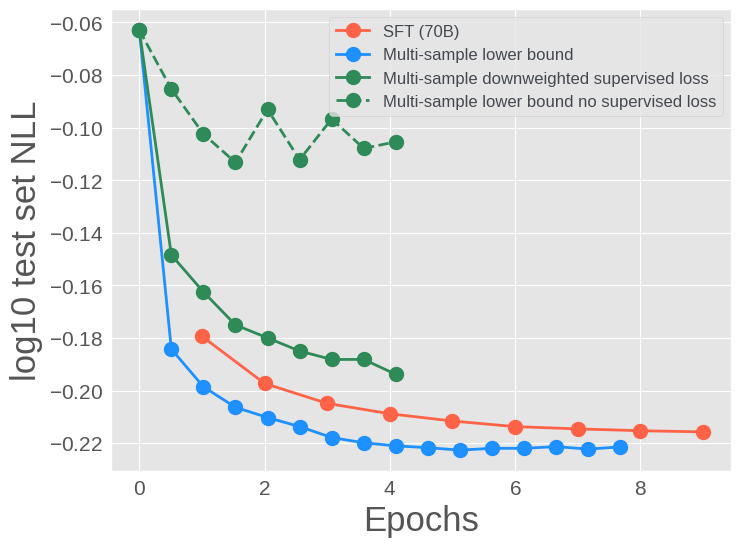}
\caption{\small{Comparison of different baselines on numina-proof test set NLL, across various algorithmic variants, with the 70B model. We observe that the supervised component of the JEPO loss plays a key role at learning efficiency and achieving good asymptotic performance.}}
\label{figure:numina-proof-ablation}
\end{figure}

To understand the role of each loss component, we carry out a few additional comparisons. Recall that JEPO update contains two parts: a REINFORCE component, whose single-sample variant is akin to VR-CLI \citep{gurung2025learning}; and a supervised loss component. We compare with a variant where the supervised loss is down-weighted ($\beta_\text{sup}=0.01$) and another where it is removed ($\beta_\text{sup}=0$).

Figure~\ref{figure:numina-proof-ablation} shows the comparison on the test set NLL. We see that by downweighting the supervised loss, JEPO makes much less progress on the test NLL given the same training epochs. Specifically, when the supervised loss is removed ($\beta_\text{sup}=0$), test NLL also seems to plateau at a worse level. Interestingly, this contrasts the observation in MATH experiments (Section~\ref{sec:exp-verifiable}) where small values of $\beta_\text{sup}$ work better. We speculate that the one difference is that nature of the chain-of-thoughts differs: for MATH or general short-form mathematical QA, the chain-of-thought details solution steps and a final answer can be readily inferred. For long-form data, the chain-of-thought tends to be a high-level outline, and it still takes extra effort to produce the full answer (e.g., proof). For the latter case, the supervised learning loss is useful.

\section{Conclusion, discussions and limitations}

We propose JEPO, a generic training paradigm scaling RL to unverifiable data, without the need for external verifiable rewards. We focus on the case where the reward is computed by matching a model generated solution 
against a dataset ground truth. We heavily draw on the probabilistic inference formulation that views chain-of-thought as latent variable. Bypassing the modeling complexity required for full ELBO, we  propose to multi-sample Jensen's evidence lower bound for scalable training. We show competitive performance on a wide array of datasets, ranging from verifiable  data like short-form math problems to unverifiable data like proof.

Possible directions for future research include studying the impact that various loss components (e.g., the REINFORCE and the supervised loss) have on overfitting; more organic ways to combine verifiable rewards and JEPO; and ways to scale JEPO in the form of meta-thought \citep{jaech2024openai,xiang2025towards} or to pre-training \citep{ruan2025reasoning}.

\bibliography{your_bib_file}
\bibliographystyle{plainnat}

\clearpage
\onecolumn

\begin{appendix}

\section{A review of the graphical model perspective}\label{appendix:graph}

We make a more extended discussion about the graphical model shown in Figure~\ref{figure:cotgraph}. 

\paragraph{Probabilistic inference with a learnable prior} Figure~\ref{figure:cotgraph}(a) shows the generic structure for probabilistic inference with a learnable prior, with latent variable $z$ and observable $o$. Here, $\theta$ controls both the prior and observation generation:
\begin{align*}
    z\sim p_\theta(\cdot), o\sim p_\theta(\cdot|c).
\end{align*}
The inference parameter $\phi$ denotes a the posterior inference distribution $q_\phi(z|o)$ that seeks to approximate the true posterior $p_\theta(z|o)\coloneqq\frac{p_\theta(c)p_\theta(o|c)}{\sum_{c'} p_\theta(c')p_\theta(o|c')}$. Together, they can form an ELBO that lower bounds the marginal log likelihood \citep{blei2017variational}
\begin{align*}
    \log p_\theta(o) \geq \underbrace{\mathbb{E}_{z\sim q_\theta(\cdot|o)}\left[\log p_\theta(o|z) + \log\frac{q_\phi(z|o)}{p_\theta(z)}\right]}_{\mathcal{L}_{\theta,\phi}(o)}.
\end{align*}
The right hand side $\mathcal{L}_{\theta,\phi}(o)$ can be optimized via stochastic gradient descent on the joint variable $(\theta,\phi)$. The lower bound is tight when the inference distribution is exactly the 
posterior $q_\phi(z|o)=p_\theta(z|o)$. A learnable prior refers to the fact that the prior distribution over latent $p_\theta(z)$ depends on $\theta$ too, while in much of the prior literature is is kept constant \citep{hoffman2016elbo,blei2017variational}. For the transition from generic probabilistic inference to our use case, a learnable prior is also fundamentally important.

\paragraph{Chain-of-thought with full ELBO} Figure~\ref{figure:cotgraph}(b) shows a direct mapping of the probabilistic inference structure to the case of optimizing chain-of-thought. Here, the chain-of-thought $c$ is the latent variable and the ground truth answer $a^*$ is the observable. A more precise mathematically definition would be to consider yet another binary optimality variable $O=\mathds{1}_{\{a=a^\ast\}}$ that determines whether the random variable answer $a$ is optimal. Here, we directly replace it with $a^*$ for notational simplicity.

If we further introduce a general conditional dependency on the prompt $x$, we arrive at the lower bound defined in Eqn~\eqref{eq:simple-lower-bound}
\begin{align*}
    \log \pi_\theta(a^*|x) \geq \underbrace{\mathbb{E}_{c\sim q_\theta(\cdot|x,a^*)}\left[\log \pi_\theta(a^*|x,c) - \log\frac{q_\phi(c|x,a^*)}{\pi_\theta(c|x)}\right]}_{\mathcal{L}_{\theta,\phi}(x,a^*)}.
\end{align*}

\paragraph{Chain-of-thought with Jensen's lower bound} In Figure~\ref{figure:cotgraph}(c), we replace the variational posterior $q_\phi$ by the prior distribution itself $\pi_\theta$. As discussed in the main paper, this looses the lower bound but make the optimization objective much simpler. See detailed derivations in Section~\ref{sec:algo}. We see there there appears to be a duplicated arrow that goes from $\theta$ to the latent variable $c$. We make such duplication to distinguish between the inference distribution (dashed arrow) and the generative distribution (solid arrow); in this particular case, we deliberately make the two distributions identical.

\paragraph{Jensen's lower bound with regularization} Finally, Figure~\ref{figure:cotgraph}(d) presents the graphical model for the case where a KL regularization is added to the Jensen's lower bound (see Lemma~\ref{lemma:regularized} for formal statements). In this case, the generative prior distribution is computed from the reference policy $\pi_\text{ref}$ parameterized by $\theta_\text{ref}$ which is kept fixed during training, while the rest of the distributions are still parameterized by $\theta$.

\section{Hyper-parameters and experimental settings} \label{appendix:hypers}

We experimented with the Llama 3.1 model of size 8B and 70B. All experiments are conducted with identical hyper-parameter settings: we always applied a batch size of $B=128$ prompts per 8B update and $B=64$ per 70B update, and sampled $n=4$ distinct generations per prompt. We found these hyper-parameters so that the model fits the GPU group memory as much as possible.

All training and evaluation sampling were conducted at a temperature of $\tau=1$ and with $\text{top-p}=1$. We did not conduct evaluation with alternative sampling configurations, in order to make training and evaluation more compatible. In our early study, deviating training sampling configuration from the above produces training instability.

For the verifiable experiments, a supervised fine-tuning on the training set was conducted to warm up the RL training, hence the beginning gap between training and test set accuracy. For the semi-verifiable and unverifiable experiments, we directly apply JEPO to the released checkpoints.

All updates were carried out with the Adam optimizer \citep{kingma2014adam} with learning rate $4\cdot 10^{-7}$. We found this learning rate by starting from a smaller value and increased the learning rate $2x$ at each iteration, to see if the training speeds up without hurting performance. We expect the results to be somehow robust to slight changes in learning rates.

Throughout all experiments (verifiable, semi-verifiable and unverifiable), for both training and evaluation, we provide system instructions that ask the model to generate a response with step-by-step solution, followed by a final conclusion phrased as \emph{the final answer is} followed by the answer predicted by the model. This is consistent with the prompt structure discussed for Llama models \citep{dubey2024llama,yue2024harp}.

\subsection{Training and evaluation dataset}

For the verifiable experiments on MATH, we train on the MATH training set with $7500$ examples and evaluate on the test set with $2500$ examples \citep{hendrycks2021measuring}. For the semi-verifiable experiments, we train with the post-processed Numina dataset \citep{numina_math_datasets}, where we split the post-processed 22k examples into a training set (90\%) and test set for evaluation. For unverifiable experiments, we extract the proof specific subset from the full Numina dataset, and split training and test set the same way as before.

\subsection{Numina dataset post-processing}

We use unverifiable proofs data from Numina 1.5 \citep{numina_math_datasets} for our experiments. We clean and filter the questions and their corresponding solutions using some simple regex heuristics. For example, we replace leading blanks, markdown headings like \texttt{\#\#}, prefixes like ``Problem:'' and ``Solution'', letter-digit combinations like ``A1'' / ``G5'' / ``ROU'', and trailing dots and blanks. After cleaning, we have 58088 proofs from the Numina dataset.

\subsection{Online RL baseline}

The online RL baseline is implemented akin to prior work such as RLOO \citep{rloo}, which can be understood as an on-policy special case of GRPO \citep{shao2024deepseekmath}. Specifically, given the verifiable reward, $r_i$, the advantage is computed with standard post-processing $A_i = \text{clip}\left((r_i - \bar{r}_{-i}) / \text{std}(r_i), -1, 1\right)$ where $\bar{r}_{-i}$ is the leave-one-out control variate. In sum, The update is
\begin{align*}
    &\frac{1}{n} \sum_{i=1}^n \left(A_i \cdot \nabla_\theta \log \pi_\theta(a_i|x,c_i)   \right) + A_i \cdot \nabla_\theta \log \pi_\theta(a_i|x,c_i)  - \beta \nabla_\theta \mathbb{KL}\left(\pi_\theta(\cdot|x), \pi_\text{ref}(\cdot|x)\right),
\end{align*}
where we intentionally separate the update to the chain-of-thought $c_i$ and sampled answer $a_i$. Juxtaposing the above update with the JEPO update, we note that the supervised loss $\nabla \log \pi_\theta(a^*|x,c_i)$ in JEPO echos the answer update $A_i \cdot \nabla_\theta \log \pi_\theta(a_i|x,c_i)$ in the RL algorithm. Formally, we show that the connection between the supervised loss and a variance reduced variant to the online RL update (Section~\ref{sec:connections}), see also Figure~\ref{figure:algorithm} for a summary of high-level comparison.

\section{Additional ablations on semi-verifiable data}\label{appendix:semi-verifiable}

\begin{figure}[t]
\centering
\includegraphics[width=\textwidth]{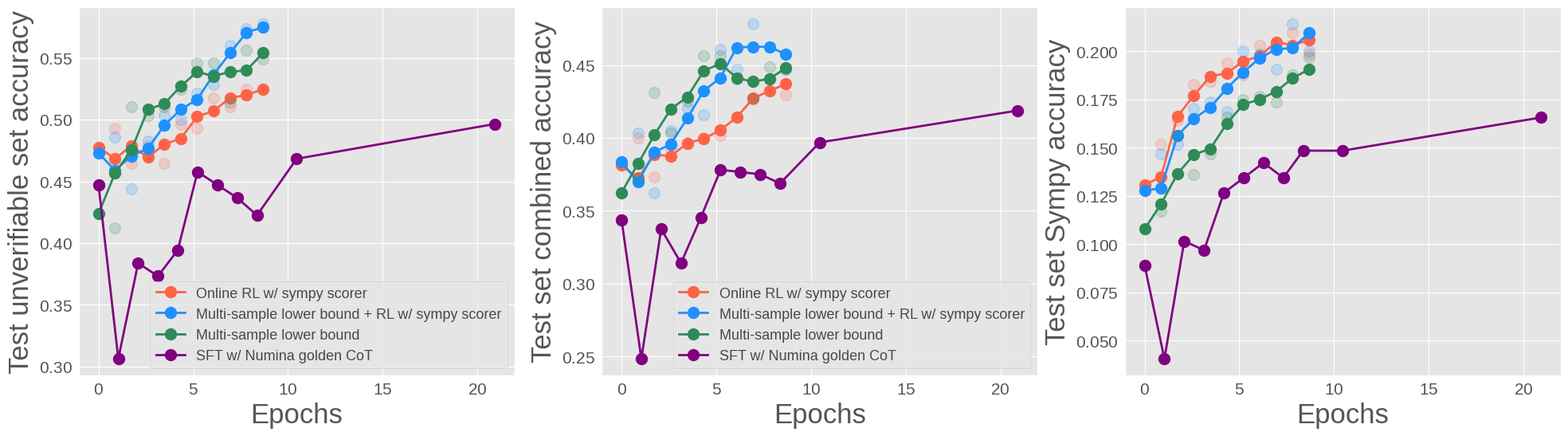}
\caption{\small{Additional comparison against a SFT baseline which trains on golden chain-of-thought data from the numina dataset. We show that the SFT baseline also improves upon various metrics, despite generally underperforming RL algorithms.}}
\label{figure:sft-semi}
\end{figure}

We show the comparison against a SFT baseline on golden chain-of-thought dataset in Figure~\ref{figure:sft-semi}. A few observations are in order: (1) SFT generally is not as good as the RL jobs, but it improves over time as we train more; (2) There is an initial drop in performance, which can be explained by the fact that the golden chain-of-thought does not conform to the familiar \emph{step-by-step} that the starting model has been post-trained with \citep{dubey2024llama}. Through SFT, the model needs to unlearn the step-by-step format and learns the more freeform hybrid format in the golden chain-of-thought data; (3) Asymptotically, SFT performs lower than RL runs.

\section{Proof of variance reduction for variance-reduced RL gradient estimate}\label{appendix:var}

Recall that we denote $(y_i)_{i=1}^n$ as the set of generations and $(c_i)_{i=1}^n$ be the set of  chain-of-thoughts generated from prompt $x$. We drop the dependency on prompt $x$
 wherever the context is clear. 
 
 \begin{proof}[Proof of Theorem~\ref{var-red}]
A direct computation shows that
\begin{eqnarray}
\begin{split}
\label{var-decomp}
 \mathbb{V}_{(y_i)_{i=1}^n\sim\pi_\theta(\cdot|x)} \left[g_\text{vanilla-pg}\right] =& \mathbb{E}_{(y_i)_{i=1}^n\sim\pi_\theta(\cdot|x)}\left[g_\text{vanilla-pg} - g_\text{var-reduced-pg} + g_\text{var-reduced-pg} - \mathbb{E}_{(y_i)_{i=1}^n\sim\pi_\theta(\cdot|x)}\left[g_\text{vanilla-pg}\right]\right] \\
 =& \mathbb{E}_{(y_i)_{i=1}^n\sim\pi_\theta(\cdot|x)}\left[\|g_\text{vanilla-pg} - g_\text{var-reduced-pg}\|^2\right] +  \mathbb{V}_{(y_i)_{i=1}^n\sim\pi_\theta(\cdot|x)} \left[g_\text{var-reduced-pg}\right],
\end{split}
\end{eqnarray}
where the cross-term vanishes due to Eqn~\eqref{cond-grad-eqn}. From this, Eqn~\eqref{var-red} follows immediately.
\end{proof}
\begin{proof}[Proof of Lemma~\ref{cond-grad-lemma}]
We begin by computing the conditional expectation $\mathbb{E}_{a\sim \pi_\theta(\cdot|c)} \left[g_\text{vanilla-pg} \; | \; (c_i)_{i=1}^n\right]$, which yields
\begin{eqnarray}
\begin{split}
 \underbrace{\mathbb{E}_{a\sim \pi_\theta(\cdot|c)}\left[\frac{1}{n}\sum_{i=1}^n \nabla_\theta\log\pi_\theta(y_i) \cdot \mathds{1}_{\{a_i=a^*\}} \; | \;  (c_i)_{i=1}^n\right]}_{\text{I}} + \underbrace{\mathbb{E}_{a\sim \pi_\theta(\cdot|c)}\left[\frac{1}{n}\sum_{i=1}^n \nabla_\theta\log\pi_\theta(y_i)\cdot \tilde{w}_i\right]}_{\text{II}}.
\end{split}
\end{eqnarray}
where we use the notation $a\sim \pi_\theta(\cdot|c)$ to indicate that each answer $a_i\sim \pi_\theta(\cdot|c_i)$ is i.i.d. sampled from its corresponding chain-of-thought. Expanding the first term $\text{I}$, we have
\begin{eqnarray}
\begin{split}
\label{term-1-expansion}
\text{I} =_{(a)}& \frac{1}{n}\sum_{i=1}^n\sum_a\big{(}\nabla_\theta\log\pi_\theta(a | c_i) + \nabla_\theta\log\pi_\theta(c_i)\big{)} \cdot \mathds{1}_{\{a=a^*\}}\cdot \pi_\theta(a|c_i)\\
=_{(b)}& \frac{1}{n}\sum_{i=1}^n \big{(}\nabla_\theta\pi_\theta(a^* | c_i) + \nabla_\theta\log\pi_\theta(c_i)\cdot\pi_\theta(a^*|c_i) \big{)},
\end{split}
\end{eqnarray}
where (a) is by definition of the expectation and $a\in\mathcal{A}$ denotes a dummy answer variable; (b) is due to the definition of the indcator function.
Now recalling the definition of $w_i$ as leave-one-out baseline to simplify term $\text{II}$:
\begin{eqnarray}
\begin{split}
\label{term-2-expansion}
\text{II} = \frac{1}{n}\sum_{i=1}^n \mathbb{E}_{a\sim \pi_\theta(\cdot|c)}\left[\nabla_\theta\log\pi_\theta(y_i)\cdot w_i \;|\; (c_i)_{i=1}^n\right]
= \frac{1}{n(n-1)}\sum_{i=1}^n \sum_{j\neq i}\mathbb{E}_{a\sim \pi_\theta(\cdot|c)}\left[\nabla_\theta\log\pi_\theta(y_i)\cdot \mathds{1}_{\{a_j=a^*\}} \; | \;  (c_i)_{i=1}^n\right].
\end{split}
\end{eqnarray}
Note we can explicitly compute each summand on the right hand side of Eqn~\eqref{term-2-expansion} as product of two conditional expectations, thanks to the fact that when $i\neq j$:
\begin{eqnarray}
\begin{split}
\label{cond-summand}
\mathbb{E}_{a\sim \pi_\theta(\cdot|c)}[\nabla_\theta\log\pi_\theta(y_i)\cdot \mathds{1}_{\{a_j=a^*\}} \; | \;  (c_i)_{i=1}^n] 
&=_{(a)} \big(\mathbb{E}_{a_i\sim\pi_\theta(\cdot|c_i)} [\nabla_\theta\log\pi_\theta(a_i | c_i) | c_i] + \nabla_\theta\log\pi_\theta(c_i)\big{)} \cdot \pi_\theta(a^*|c_j) \\
&=_{(b)} \nabla_\theta\log\pi_\theta(c_i) \cdot  \pi_\theta(a^*|c_j),
\end{split}
\end{eqnarray}
where (a) is due to the definition of the indicator function; (b) is based on the zero-mean property of score functions. Plugging Eqn~\eqref{cond-summand} into the right hand side of Eqn~\eqref{term-2-expansion}, we have
\begin{eqnarray}
\begin{split}
\label{term-2-continue}
\text{II} 
= \frac{1}{n}\sum_{i=1}^n \nabla_\theta\log\pi_\theta(c_i)\cdot \frac{1}{n-1}\sum_{j\neq i}\pi_\theta(a^*|c_j)
= \frac{1}{n}\sum_{i=1}^n \nabla_\theta\log\pi_\theta(c_i) \cdot \tilde{w}_i,
\end{split}
\end{eqnarray}
where we used the definition of $\tilde{w}_i$ from Eqn~\eqref{cond-rloo-grad}. Lastly, we combine Eqn~\eqref{term-1-expansion} and Eqn~\eqref{term-2-continue} and obtain
\begin{eqnarray}
\begin{split}
\text{I} + \text{II} = \frac{1}{n}\sum_{i=1}^n \bigg{(}\nabla_\theta\pi_\theta(a^* | c_i) + \nabla_\theta\log\pi_\theta(c_i) \cdot \big{(}\pi_\theta(a^*|c_i) - \tilde{w}_i \big{)}\bigg{)} = g_\text{var-reduced-pg}.
\end{split}
\end{eqnarray}
Thus we have concluded the proof of Lemma~\ref{cond-grad-lemma}.
\end{proof}

\end{appendix}

\end{document}